\documentclass[10pt,twocolumn,letterpaper]{article}

\usepackage{cvpr}
\usepackage{times}
\usepackage{epsfig}
\usepackage{graphicx}
\usepackage{amsmath}
\usepackage{amssymb}
\usepackage[pagebackref=true,breaklinks=true,letterpaper=true,colorlinks,bookmarks=false]{hyperref}
\usepackage{url}
\usepackage{amsthm}
\usepackage{amsfonts}    
\usepackage{nicefrac}       
\usepackage{microtype}      
\usepackage{xcolor}
\usepackage{grffile}
\usepackage{verbatim}
\usepackage{setspace}
\usepackage[format=plain,labelformat=simple,labelsep=period,font=small,skip=4pt,compatibility=false]{caption}
\usepackage[font=scriptsize,skip=2pt]{subcaption}
\usepackage{booktabs,tabularx}
\usepackage[inline]{enumitem}
\usepackage{mathtools}
\usepackage[ruled,linesnumbered]{algorithm2e}
\usepackage[capitalize]{cleveref}
\usepackage{etoolbox,siunitx}
\robustify\bfseries
\robustify\itshape
\usepackage[keeplastbox]{flushend}
\usepackage{newtxtt}
\usepackage{appendix}

\graphicspath{ {./Images/} }
\newcommand{\myparagraph}[1]{\smallskip\noindent\textbf{#1}}
\newtheorem{theorem}{Theorem}[section]

\newtheorem{lemma}[theorem]{Lemma}
\newcommand{\squeeze}{\textsc{Squeeze}}
\newcommand{\osplit}{\textsc{Split}}
\newcommand{\marps}{\textsc{Marps}}
\newcommand{\sflow}{\textsc{StepOfFlow}}
\newcommand{\squeezeinv}{\textsc{SqueezeInverse}}
\newcommand{\osplitinv}{\textsc{SplitInverse}}
\newcommand{\sflowinv}{\textsc{StepOfFlowInverse}}

\newcommand{\iid}{i.i.d\onedot}

\cvprfinalcopy 

\def\cvprPaperID{8386} 

\usepackage{fancyhdr}
\usepackage{setspace}

\begin{document}

\title{Normalizing Flows with Multi-Scale Autoregressive Priors}

\author{Shweta Mahajan\footnotemark[1] $^{\;1}$ \quad Apratim Bhattacharyya\footnotemark[1] $^{\;2}$  \quad  Mario Fritz$^{3}$ \quad Bernt Schiele$^2$ \quad Stefan Roth$^1$\\[0.5em]
$^1$Department of Computer Science, TU Darmstadt\\[0.10em]
$^2$Max Planck Institute for Informatics, Saarland Informatics Campus\\[0.10em]
$^3$CISPA Helmholtz Center for Information Security, Saarland Informatics Campus\\
}


\maketitle
\thispagestyle{fancy}

\begin{abstract}
\footnotetext[1]{Authors contributed equally.}
Flow-based generative models are an important class of exact inference models that admit efficient inference and sampling for image synthesis. 
Owing to the efficiency constraints on the design of the flow layers, \eg~split coupling flow layers in which approximately half the pixels do not undergo further transformations, they have limited expressiveness for modeling long-range data dependencies compared to autoregressive models that rely on conditional pixel-wise generation. 
In this work, we improve the representational power of flow-based models by introducing channel-wise dependencies in their latent space through multi-scale autoregressive priors (mAR). 
Our mAR prior for models with split coupling flow layers (mAR-SCF) can better capture dependencies in complex multimodal data. 
The resulting model achieves state-of-the-art density estimation results on MNIST, CIFAR-10, and ImageNet.
Furthermore, we show that mAR-SCF allows for improved image generation quality, with gains in FID and Inception scores compared to state-of-the-art flow-based models. 

\end{abstract}

\section{Introduction}
Deep generative models aim to learn complex dependencies within very high-dimensional input data, \eg~natural images \cite{BrockDS19,abs-1906-00446} or audio data \cite{DielemanOS18}, and enable generating new samples that are representative of the true data distribution. 
These generative models find application in various downstream tasks like image synthesis \cite{GoodfellowPMXWOCB14,KingmaW13,OordKK16} or speech synthesis \cite{DielemanOS18,OordLBSVKDLCSCG18}.
Since it is not feasible to learn the exact distribution, generative models generally approximate the underlying true distribution.
Popular generative models for capturing complex data distributions are Generative Adversarial Networks (GANs) \cite{GoodfellowPMXWOCB14}, which model the distribution implicitly and generate (high-dimensional) samples by transforming a noise distribution into the desired space with complex dependencies; however, they may not cover all modes of the underlying data distribution. 
Variational Autoencoders (VAEs) \cite{KingmaW13} optimize a lower bound on the log-likelihood of the data. 
This implies that VAEs can only approximately optimize the log-likelihood \cite{RezendeMW14}.

Autoregressive models \cite{domke2008killed,OordKEKVG16,OordKK16} and normalizing flow-based generative models \cite{DinhKB14,DinhSB17,KingmaD18} are exact inference models that optimize the exact log-likelihood of the data.
Autoregressive models can capture complex and long-range dependencies between the dimensions of a distribution, \eg~in case of images, as the value of a pixel is conditioned on a large context of neighboring pixels. 
The main limitation of this approach is that image synthesis is sequential and thus difficult to parallelize. 
\begin{figure}[t!]
  \centering
  \includegraphics[width=1\linewidth]{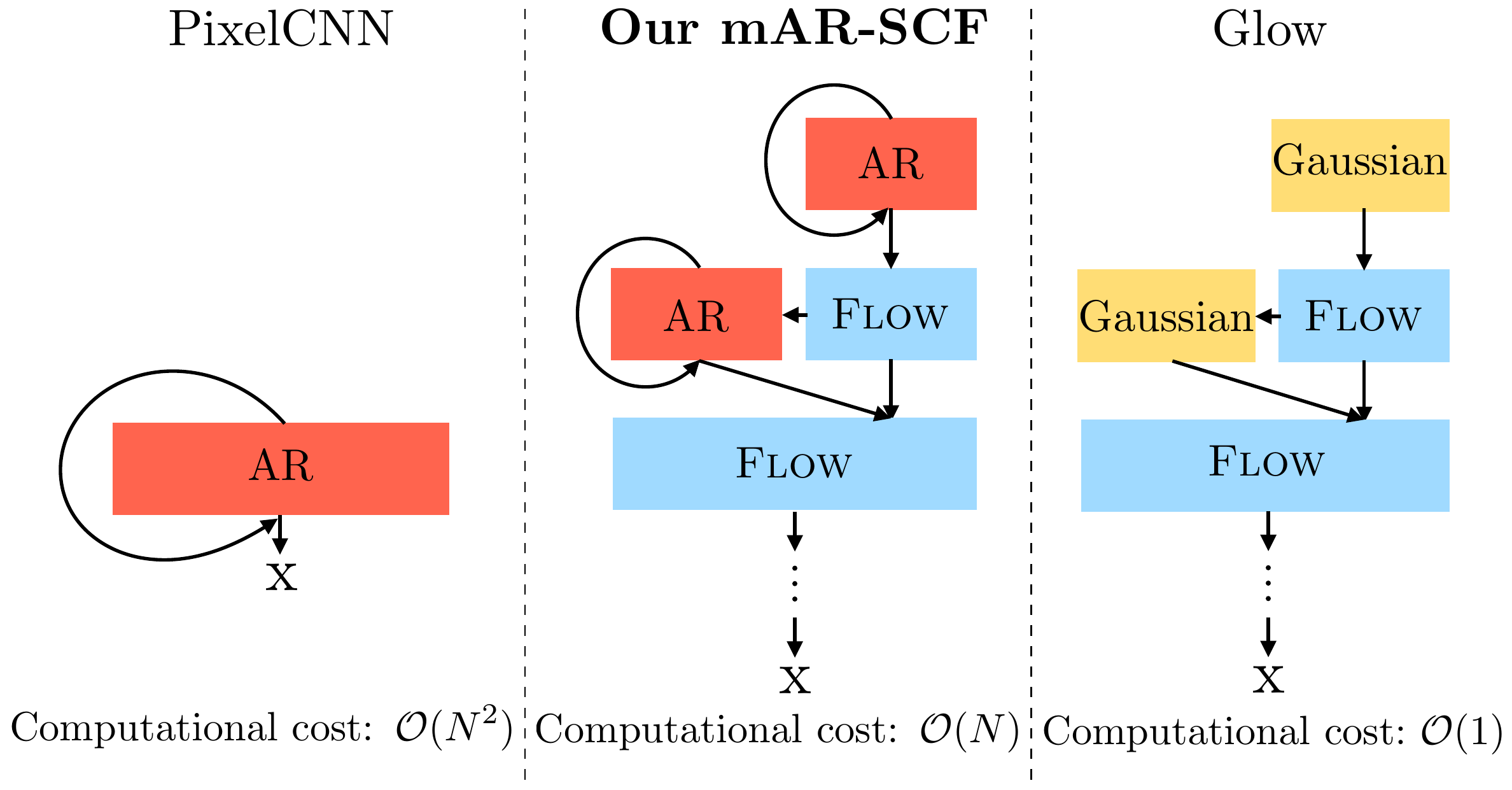}
\caption{Our \emph{mAR-SCF} model combines normalizing flows with autoregressive (AR) priors to improve modeling power while ensuring that the computational cost grows linearly with the spatial image resolution $N\times N$.}
\label{fig:teaser} 
\vspace{-0.5em}
\end{figure}
Recently proposed normalizing flow-based models, such as NICE \cite{DinhKB14}, RealNVP \cite{DinhSB17}, and Glow \cite{KingmaD18}, allow exact inference by mapping the input data to a known base distribution, \eg~a Gaussian, through a series of invertible transformations. 
These models leverage invertible split coupling flow (SCF) layers in which certain dimensions are left unchanged by the invertible transformation as well as \osplit{} operations following which certain dimensions do not undergo subsequent transformations. 
This allows for considerably easier parallelization of both inference and generation processes.
However, these models lag behind autoregressive models for density estimation. 

In this work, we 
\emph{(i)} propose multi-scale autoregressive priors for invertible flow models with split coupling flow layers, termed \emph{mAR-SCF}, to address the limited modeling power of non-autoregressive invertible flow models \cite{DinhSB17,HoCSDA19,KingmaD18,abs-1906-00446}  (\cref{fig:teaser});
\emph{(ii)} we apply our multi-scale autoregressive prior after every \osplit{} operation such that the computational cost of sampling grows linearly in the spatial dimensions of the image compared to the quadratic cost of traditional autoregressive models (given sufficient parallel resources);
\emph{(iii)} our experiments show that we achieve state-of-the-art density estimation results on MNIST \cite{lecun1998gradient}, CIFAR10 \cite{krizhevsky2009learning}, and ImageNet \cite{Russakovsky:2015:ILS} compared to prior invertible flow-based approaches; and finally
\emph{(iv)} we show that our multi-scale autoregressive prior leads to better sample quality as measured by the FID metric \cite{HeuselRUNH17} and the Inception score \cite{SalimansGZCRCC16}, significantly lowering the gap to GAN approaches \cite{RadfordMC15,WeiGL0W18}. 

\section{Related Work}
In this work, we combine the expressiveness of autoregressive models with the efficiency of non-autoregressive flow-based models for exact inference.

\myparagraph{Autoregressive models} \cite{ChenMRA18,domke2008killed,Graves13,HochreiterS97,ParmarVUKSKT18,OordKEKVG16,OordKK16} are a class of exact inference models that factorize the joint probability distribution over the input space as a product of conditional distributions, where each dimension is conditioned on the previous ones in a pre-defined order. 
Recent autoregressive models, such as PixelCNN and PixelRNN \cite{OordKEKVG16,OordKK16}, can generate high-quality images but are difficult to parallelize since synthesis is sequential. 
It is worth noting that autoregressive image models, such as that of Domke \etal~\cite{domke2008killed}, significantly pre-date their recent popularity.
Various extensions have been proposed to improve the performance of the PixelCNN model. 
For example,  Multiscale-PixelCNN \cite{ReedOKCWCBF17} extends PixelCNN to improve the sampling runtime from linear to logarithmic in the number of pixels, exploiting conditional independence between the pixels.
Chen \etal~\cite{ChenMRA18} introduce self-attention in PixelCNN models to improve the modeling power.
Salimans \etal~\cite{SalimansK0K17} introduce skip connections and a discrete logistic likelihood model.   
WaveRNN \cite{KalchbrennerESN18} leverages customized GPU kernels to improve the sampling speed for audio synthesis.  Menick \etal~\cite{MenickK19} synthesize images by sequential conditioning on sub-images within an image.
These methods, however, still suffer from slow sampling speed and are difficult to parallelize. 

\myparagraph{Flow-based generative models}, first introduced in \cite{DinhKB14}, also allow for exact inference. 
These models are composed of a series of invertible transformations, each with a tractable Jacobian and inverse, which maps the input distribution to a known base density, \eg~a Gaussian. 
Papamakarios \etal~\cite{PapamakariosMP17} proposed autoregressive invertible transformations using masked decoders. However, these are difficult to parallelize just like PixelCNN-based approaches.
Kingma \etal~\cite{kingma2016improved} propose inverse autoregressive flow (IAF), where the means and variances of pixels depend on random variables and not on previous pixels, making it easier to parallelize. 
However, the approach offers limited generalization \cite{OordLBSVKDLCSCG18}.

\myparagraph{Recent extensions.}
Recent work \cite{BehrmannGCDJ19,DinhSB17,KingmaD18} extends normalizing flows \cite{DinhKB14} to multi-scale architectures with split couplings, which allow for efficient inference and sampling.
For example, Kingma \etal~\cite{KingmaD18} introduce additional invertible $1 \times 1$ convolutions to capture non-linearities in the data distribution. 
Hoogeboom \etal~\cite{HoogeboomBW19} extend this to $d \times d$ convolutions, increasing the receptive field. 
\cite{abs-1906-02735} improve the residual blocks of flow layers with memory efficient gradients based on the choice of activation functions. 
A key advantage of flow-based generative models is that they can be parallelized for inference and synthesis. 
Ho \etal~\cite{HoCSDA19} propose Flow++ with various modifications in the architecture of the flows in \cite{DinhSB17}, including attention and a variational quantization method to improve the data likelihood. 
The resulting model is computationally expensive as non-linearities are applied along all the dimensions of the data at every step of the flow, \ie all the dimensions are instantiated with the prior distribution at the last layer of the flow.
While comparatively efficient, such flow-based models have limited expressiveness compared to autoregressive models, which is reflected in their lower data log-likelihood. 
It is thus desirable to develop models that have the expressiveness of autoregressive models and the efficiency of flow-based models.
This is our goal here.

\myparagraph{Methods with complex priors.}
Recent work \cite{0022KSDDSSA17} develops complex priors to improve the data likelihoods.
VQ-VAE2 integrates autoregressive models as priors \cite{abs-1906-00446} with discrete latent variables \cite{0022KSDDSSA17} for high-quality image synthesis and proposes latent graph-based models in a VAE framework.  
Tomczak \etal~\cite{TomczakW18} propose mixtures of Gaussians with predefined clusters, and \cite{0022KSDDSSA17} use neural autoregressive model priors in the latent space, which improves results for image synthesis.
Ziegler \etal~\cite{ZieglerR19} learn a prior based on normalizing flows to capture multimodal discrete distributions of character-level texts in the latent spaces with nonlinear flow layers. 
However, this invertible layer is difficult to be optimized in both directions.  
Moreover, these models do not allow for exact inference.
In this work, we propose complex autoregressive priors to improve the power of invertible split coupling-based normalizing flows \cite{DinhSB17,HoCSDA19,KingmaD18}.
\begin{figure*}[t!]
\begin{subfigure}[b]{0.38\linewidth}
\centering
   \includegraphics[width=0.64\linewidth]{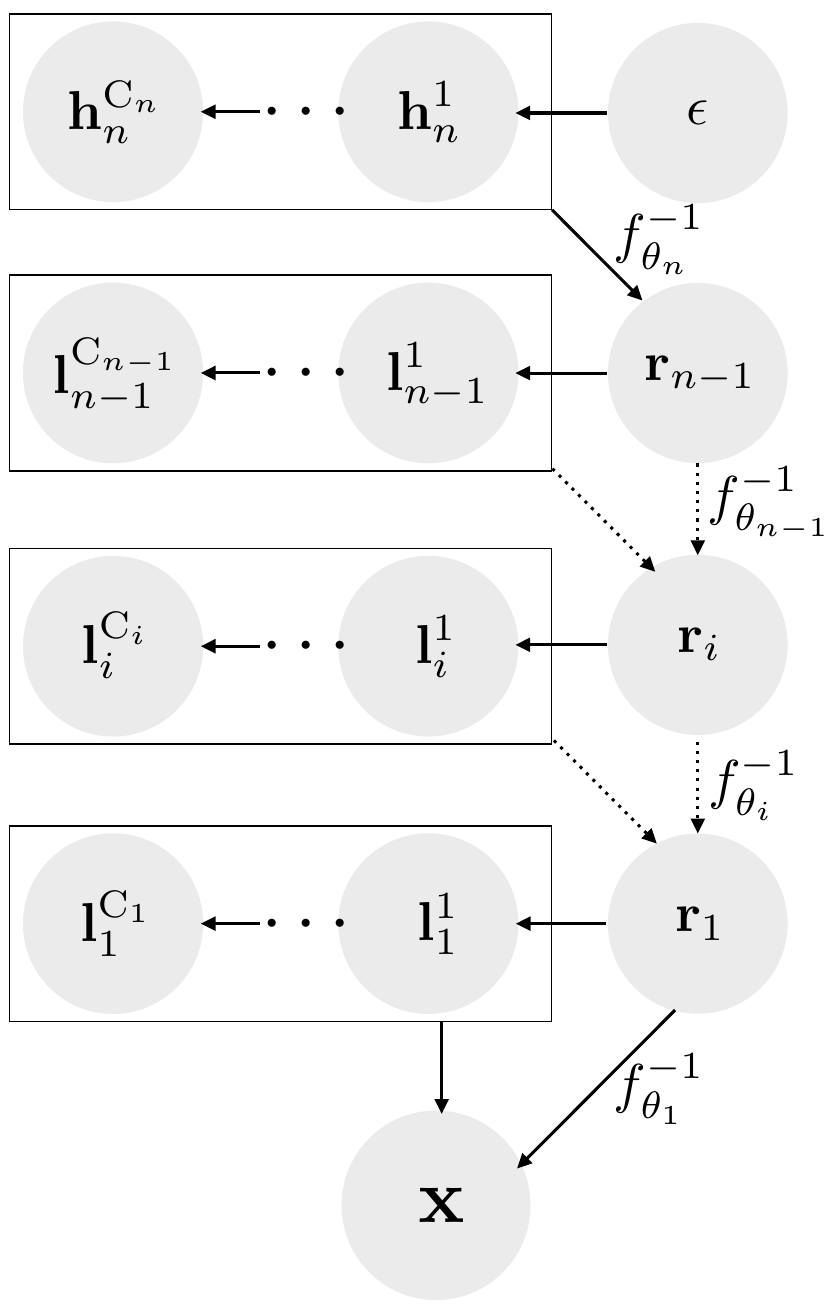}
   \caption{Generative model for mAR-SCF.}
   \label{fig:gen_model}
\end{subfigure}%
\begin{subfigure}[b]{0.62\linewidth}
\centering
  \includegraphics[width=\linewidth]{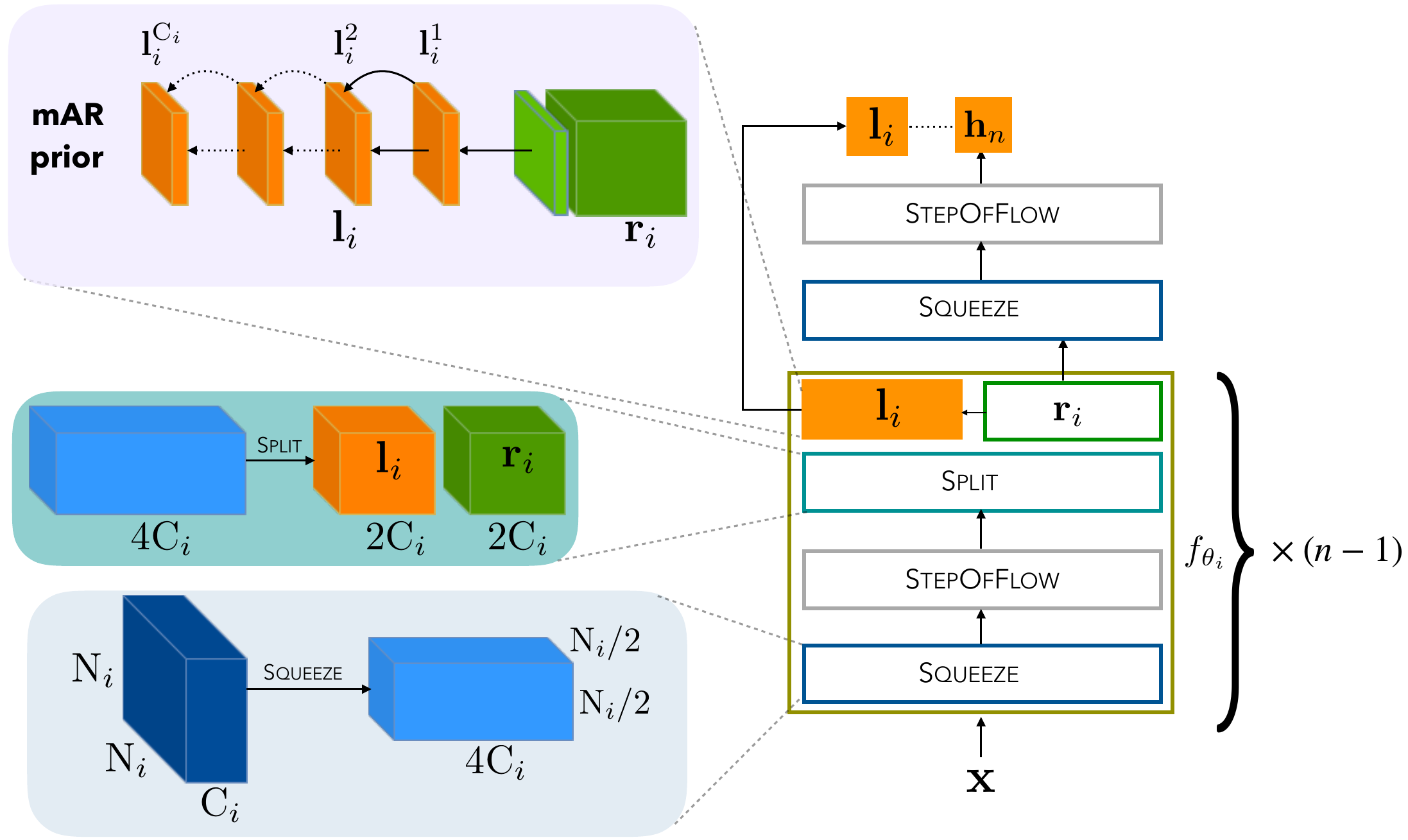}
  \caption{SCF flow with mAR prior.}
  \label{fig:marscf}
\end{subfigure}
\caption{Flow-based generative models with multi-scale autoregressive priors (\emph{mAR-SCF}). The generative model (\emph{left}) shows the multi-scale autoregressive sampling of the channel dimensions of $\mathbf{l}_i$ at each level. The spatial dimensions of each channel are sampled in parallel. $\mathbf{r}_i$ are computed with invertable transformations.  The \emph{mAR-SCF} model (\emph{right}) shows the complete multi-scale architecture with the mAR prior applied along the channels of $\mathbf{l}_i$, \ie at each level $i$ after the \osplit{} operation.}\label{fig:archi}
\end{figure*}

\section{Overview and Background}
In this work, we propose multi-scale autoregressive priors for split coupling-based flow models, termed \emph{mAR-SCF}, where we leverage autoregressive models to improve the modeling flexibility of invertible normalizing flow models without sacrificing sampling efficiency. 
As we build upon normalizing flows and autoregressive models, we first provide an overview of both. 

\myparagraph{Normalizing flows.}
Normalizing flows \cite{DinhKB14} are a class of exact inference generative models. 
Here, we consider invertible flows, which allow for both efficient exact inference and sampling. 
Specifically, invertible flows consist of a sequence of $n$ invertible functions $f_{\theta_i}$, which transform a density on the data $\mathbf{x}$ to a density on latent variables $\mathbf{z}$,
\begin{align}
\mathbf{x} \overset{f_{\theta_1}}{\longleftrightarrow} \mathbf{h}_{1} \overset{f_{\theta_2}}{\longleftrightarrow} \mathbf{h}_{2} \cdots \overset{f_{\theta_n}}{\longleftrightarrow} \mathbf{z}.
\label{eq:unflowtrans}
\end{align}

Given that we can compute the likelihood of $p(\mathbf{z})$, the likelihood of the data $\mathbf{x}$ under the transformation $f$ can be computed using the change of variables formula,
\begin{align}
\log p_{\theta}(\mathbf{x}) = \log p(\mathbf{z}) + \sum\limits_{i=1}^{n} \log \;\lvert \det J_{\theta_i} \lvert,
\label{eq:flow}
\end{align}
where $J_{\theta_i} = \nicefrac{\partial \mathbf{h}_{i}}{\partial \mathbf{h}_{i-1}}$ is the Jacobian of the invertible transformation $f_{\theta_i}$ going from $\mathbf{h}_{i-1}$ to $\mathbf{h}_{i}$ with $\mathbf{h}_0\equiv\mathbf{x}$.
Note that most prior work \cite{abs-1906-02735,DinhKB14,DinhSB17,HoCSDA19,KingmaD18} considers \iid Gaussian likelihood models of $\mathbf{z}$, \eg~$p(\mathbf{z}) = \mathcal{N}(\mathbf{z}\,|\,\mu,\sigma)$.

These models, however, have limitations. 
First, the requirement of invertibility constrains the class of functions $f_{\theta_i}$ to be monotonically increasing (or decreasing), thus limiting expressiveness. 
Second, of the three possible variants of $f_{\theta_i}$ to date \cite{ZieglerR19}, MAF (masked autoregressive flows), IAF (inverse autoregressive flows), and SCF (split coupling flows), MAFs are difficult to parallelize due to sequential dependencies between dimensions and IAFs do not perform well in practice. SCFs strike the right balance with respect to parallelization and modeling power. In detail, SCFs partition the dimensions into two equal halves \emph{and} transform one of the halves $\mathbf{r}_i$ conditioned on $\mathbf{l}_i$, leaving $\mathbf{l}_i$ unchanged and thus not introducing any sequential dependencies (making parallelization easier). Examples of SCFs include the affine couplings of RealNVP \cite{DinhSB17} and MixLogCDF couplings of Flow++ \cite{HoCSDA19}.


In practice, SCFs are organized into blocks \cite{DinhKB14,DinhSB17,KingmaD18} to maximize efficiency such that each $f_{\theta_i}$ typically consists of \squeeze{}, \sflow{}, and \osplit{} operations. 
\squeeze{} trades off spatial resolution for channel depth.
Suppose an intermediate layer $\mathbf{h}_{i}$ is of size $\left[\text{C}_i,\text{N}_i,\text{N}_i\right]$, then the \squeeze{} operation transforms it into size $\left[4 \, \text{C}_i,\text{N}_i/2,\text{N}_i/2\right]$ by reshaping $2\times 2$ neighborhoods into $4$ channels.
\sflow{} is a series of SCF (possibly several) coupling layers and invertible $1 \times 1$ convolutions \cite{DinhKB14,DinhSB17,KingmaD18}.

The \osplit{} operation (distinct from the split couplings) splits an intermediate layer $\mathbf{h}_{i}$ into two halves $\left\{\mathbf{l}_{i},\mathbf{r}_{i}\right\}$ of size $\left[2 \, \text{C}_i,\text{N}_i/2,\text{N}_i/2\right]$ each. Subsequent invertible layers ${f}_{\theta{{j}>{i}}}$ operate only on $\mathbf{r}_{i}$, leaving $\mathbf{l}_{i}$ unchanged. 
In other words, the \osplit{} operation fixes some dimensions of the latent representation $\mathbf{z}$ to $\mathbf{l}_{i}$ as they are not transformed any further. 
This leads to a significant reduction in the amount of computation and memory needed. 
In the following, we denote the spatial resolutions at the $n$ different levels as $\text{N}=\{\text{N}_{0}, \cdots, \text{N}_n\}$, with $N=\text{N}_0$ being the input resolution.
Similarly, $\text{C}=\{\text{C}_{0}, \cdots, \text{C}_n\}$ denotes the number of feature channels, with $C=\text{C}_0$ being the number of input channels.

In practice, due to limited modeling flexibility, prior SCF-based models \cite{DinhSB17,HoCSDA19,KingmaD18} require many SCF coupling layers in $f_{\theta_i}$ to model complex distributions, \eg~images. 
This in turn leads to high memory requirements and also leads to less efficient sampling procedures.

\myparagraph{Autoregressive models.}
Autoregressive generative models are another class of powerful and highly flexible exact inference models.
They factorize complex target distributions by decomposing them into a product of conditional distributions, \eg~images with $N \times N$ spatial resolution as
$p(\mathbf{x}) = \prod_{i=1}^N\prod_{j=1}^N p(\mathbf{x}_{i,j}| \mathbf{x}_{\text{pred}(i,j)})$ \cite{domke2008killed,Graves13,PapamakariosMP17,OordKEKVG16,OordKK16}.
Here, $\text{pred}(i,j)$ denotes the set of predecessors of pixel $(i,j)$.
The functional form of these conditionals can be highly flexible, and allows such models to capture complex multimodal distributions. However, such a dependency structure only allows for image synthesis via ancestral sampling by generating each pixel sequentially, conditioned on the previous pixels \cite{OordKEKVG16,OordKK16}, making parallelization difficult. This is also inefficient since autoregressive models, including PixelCNN and PixelRNN, require $\mathcal{O}(N^2)$ time steps for sampling.

\section{Multi-scale Autoregressive Flow Priors}
We propose to leverage the strengths of autoregressive models to improve invertible normalizing flow models such as \cite{DinhSB17,KingmaD18}. 
Specifically, we propose novel \emph{multi-scale autoregressive priors for split coupling flows (mAR-SCF)}. 
Using them allows us to learn complex multimodal latent priors $p(\mathbf{z})$ in multi-scale SCF models, \cf \cref{eq:flow}.
This is unlike \cite{DinhSB17,HoCSDA19,KingmaD18,abs-1906-00446}, which rely on Gaussian priors in the latent space. 
Additionally, we also propose a scheme for interpolation in the latent space of our \emph{mAR-SCF} models.

The use of our novel autoregressive \emph{mAR} priors for invertible flow models has two distinct advantages over both vanilla SCF and autoregressive models. First, the powerful autoregressive prior helps mitigate the limited modeling capacity of the vanilla SCF flow models. 
Second, as only the prior is autoregressive, this makes flow models with our \emph{mAR} prior an order of magnitude faster with respect to sampling time than fully autoregressive models. 
Next, we describe our multi-scale autoregressive prior in detail.

Our \emph{mAR-SCF} model uses an efficient invertible split coupling flow $f_{\theta_i}(\mathbf{x})$ to map the distribution over the data $\mathbf{x}$ to a latent variable $\mathbf{z}$ and then models an autoregressive  \emph{mAR} prior over $\mathbf{z}$, parameterized by $\phi$. 
The likelihood of a data point $\mathbf{x}$ of dimensionality  $\left[C,N,N\right]$ can be expressed as
\begin{align}
\log p_{\theta,\phi}(\mathbf{x}) = \log p_{\phi}(\mathbf{z}) + \sum\limits_{i=1}^{n} \log \;\lvert \det J_{\theta_i} \lvert.
\label{eq:hflow}
\end{align}

Here, $J_{\theta_i}$ is the Jacobian of the invertible transformations $f_{\theta_i}$.
Note that, as $f_{\theta_i}(\mathbf{x})$ is an invertible function, $\mathbf{z}$ has the same total dimensionality as the input data point $\mathbf{x}$.

\myparagraph{Formulation of the \emph{mAR} prior.}
We now introduce our \emph{mAR} prior $p_\phi(\mathbf{z})$ along with our \emph{mAR-SCF} model, which combines the split coupling flows $f_{\theta_i}$ with an \emph{mAR} prior.
As shown in \cref{fig:archi}, our \emph{mAR} prior is applied after every \osplit{} operation of the invertible flow layers as well as at the smallest spatial resolution. 
Let $\mathbf{l}_{i} = \big\{ \mathbf{l}^{1}_{i}, \cdots, \mathbf{l}^{{\text{C}_i}}_{i}\big\}$ be the $\text{C}_i$ channels of size $[\text{C}_i,\text{N}_i,\text{N}_i]$, which do not undergo further transformation $f_{\theta_i}$ after the \osplit{} at level $i$. 
Following the \osplit{} at level $i$, our \emph{mAR} prior is modeled as a conditional distribution, $p_{\phi}(\mathbf{l}_{i} | \mathbf{r}_{i})$; at the coarsest spatial resolution it is an unconditional distribution, $p_{\phi}(\mathbf{h}_{n})$. Thereby, we assume that our \emph{mAR} prior at each level $i$ autoregressively factorizes along the channel dimension as
\begin{align}
 p_{\phi}(\mathbf{l}_i|\mathbf{r}_i)=\prod_{j=1}^{\text{C}_i}p_{\phi}\Big(\mathbf{l}^{j}_{i} \Big| \mathbf{l}^{1}_{i}, \cdots, \mathbf{l}^{j-1}_{i}, \mathbf{r}_{i}\Big).
 \label{eq:condar}
\end{align}
Furthermore, the distribution at each spatial location $\left(m,n\right)$ within a channel $\mathbf{l}^{j}_{i}$ is modeled as a conditional Gaussian,
\begin{align}
p_{\phi}({l}^{j}_{i(m,n)} | \mathbf{l}^{1}_{i}, \cdots, \mathbf{l}^{j-1}_{i}, \mathbf{r}_{i} ) = \mathcal{N}\Big(\mu^{j}_{i(m,n)}, \sigma^{j}_{i(m,n)} \Big).
\label{eq:condar_g}
\end{align}
Thus, the mean, $\mu^{j}_{i(m,n)}$ and variance, $\sigma^{j}_{i(m,n)}$ at each spatial location are autoregressively modeled along the channels. This allows the distribution at each spatial location to be highly flexible and capture multimodality in the latent space. Moreover from \cref{eq:condar}, our \emph{mAR} prior can model long-range correlations in the latent space as the distribution of each channel is dependent on all previous channels. 

This autoregressive factorization allows us to employ Conv-LSTMs \cite{ShiGL0YWW17} to model the  distributions $p_{\phi}(\mathbf{l}^{j}_{i}| \mathbf{l}^{1}_{i}, \cdots, \mathbf{l}^{j-1}_{i}, \mathbf{r}_{i})$ and $p_{\phi}(\mathbf{h}_{n})$. Conv-LSTMs can model long-range dependencies across channels in their internal state. Additionally, long-range spatial dependencies within channels can be modeled by stacking multiple Conv-LSTM layers with a wide receptive field. This formulation allows all pixels within a channel to be sampled in parallel, while the channels are sampled in a sequential manner,
\begin{align}
\hat{\mathbf{l}}^{j}_{i} \sim p_{\phi}\Big(\mathbf{l}^{j}_{i} \Big| \mathbf{l}^{1}_{i}, \cdots, \mathbf{l}^{j-1}_{i}, \mathbf{r}_{i}\Big).
\label{eq:condarsamp}
\end{align}
This is in contrast to PixelCNN/RNN-based models, which sample one pixel at a time.

\myparagraph{The \emph{mAR-SCF} model.} We illustrate our \emph{mAR-SCF} model architecture in \cref{fig:marscf}. Our \emph{mAR-SCF} model leverages the \squeeze{} and \osplit{} operations for invertible flows introduced in \cite{DinhKB14,DinhSB17} for efficient parallelization. 
Following \cite{DinhKB14,DinhSB17,KingmaD18}, we use several \squeeze{} and \osplit{} operations in a multi-scale setup at $n$ scales (\cref{fig:marscf}) until the spatial resolution at $\mathbf{h}_{n}$ is reasonably small, typically $4 \times 4$. 
Note that there is no \osplit{} operation at the smallest spatial resolution. 
Therefore, the latent space is the concatenation of $\mathbf{z}=\{\mathbf{l}_{1}, \ldots, \mathbf{l}_{n-1},\mathbf{h}_{n}\}$.
The split coupling flows (SCF) $f_{\theta_i}$ in the \emph{mAR-SCF} model remain invertible by construction.
We consider different SCF couplings for $f_{\theta_i}$, including the affine couplings of \cite{DinhSB17,KingmaD18} and MixLogCDF couplings \cite{HoCSDA19}.

Given the parameters $\phi$ of our multimodal \emph{mAR} prior modeled by the Conv-LSTMs, we can compute $p_{\phi}(\mathbf{z})$ using the formulation in \cref{eq:condar,eq:condar_g}.
We can thus express \cref{eq:hflow} in \emph{closed form} and directly maximize the likelihood of the data under the multimodal \emph{mAR} prior distribution learned by the Conv-LSTMs. 

Next, we show that the computational cost of our \emph{mAR-SCF} model is $\mathcal{O}(N)$ for sampling an image of size $\left[C,N,N\right]$; this is in contrast to the standard $\mathcal{O}(N^2)$ computational cost required by purely autoregressive models. 

\begin{algorithm}[t]
 Sample $\hat{\mathbf{h}}_{n} \sim p_{\phi}(\mathbf{h}_{n})$ \;
 \For{$i\gets n-1$ \KwTo $1$}{
    \tcc{$\text{\osplitinv{}}$}    
        $\hat{\mathbf{r}}_{i} \gets \hat{\mathbf{h}}_{i+1}$ \tcp*{Assign previous}\
        $\hat{\mathbf{l}}_{i} \sim p_{\phi}(\mathbf{l}_{i} | \mathbf{r}_{i})$ \tcp*{ Sample \emph{mAR} prior}\
        $\hat{\mathbf{h}}_{i} \gets \Big\{ \hat{\mathbf{l}}_{i} , \hat{\mathbf{r}}_{i} \Big\}$ \tcp*{Concatenate}
    \tcc{$\text{\sflowinv{}}$}
        Apply $f^{-1}_{i}(\hat{\mathbf{h}}_{i})$ \tcp*{SCF coupling} 
    \tcc{$\text{\squeezeinv{}}$}
        Reshape $\hat{\mathbf{h}}_{i}$  \tcp*{Depth to Space}
 }
 $\mathbf{x} \gets \hat{\mathbf{h}}_{1}$ \;
 \caption{\marps{}: Multi-scale Autoregressive Prior Sampling for our \emph{mAR-SCF} models}
 \label{algo:samp}
\end{algorithm}

\begin{table*}[t]
  \centering
  \footnotesize
  \begin{tabularx}{\textwidth}{@{}Xl@{\hskip 3em}S[table-format=1]S[table-format=2]S[table-format=3]S[table-format=1.2]@{\hskip 3em}S[table-format=1]S[table-format=2]S[table-format=3]S[table-format=1.2]@{}}
	\toprule
	& & \multicolumn{4}{c@{\hskip 2em}}{MNIST} & \multicolumn{4}{@{\hskip 1em}c@{}}{CIFAR10}\\
	\cmidrule(l{-0.1em}r{3.1em}){3-6} \cmidrule(l{-0.1em}){7-10}
	Method & Coupling & {Levels} & {$\lvert \text{SCF} \lvert$} & {Channels} & {{\bfseries bits/dim} $(\downarrow)$} & {Levels} & {$\lvert \text{SCF} \lvert$} & {Channels} & {{\bfseries bits/dim} $(\downarrow)$} \\
	\midrule
    PixelCNN \cite{OordKK16} & Autoregressive & {--} & {--} & {--} & {--} & {--} & {--} & {--} & 3.00\\
	PixelCNN++ \cite{OordKEKVG16} & Autoregressive & {--} & {--} & {--} & {--} & {--} & {--} & {--} & \itshape 2.92\\
	\midrule
	Glow \cite{KingmaD18} & Affine & 3 & 32 & 512 & 1.05 & 3 & 32 & 512 & 3.35  \\ 
	Flow++ \cite{HoCSDA19} & MixLogCDF & {--} & {--} & {--} & {--}  & 3 & {--} & 96 & 3.29 \\
	Residual Flow \cite{abs-1906-02735} & Residual & 3 & 16 & {--} & 0.97 & 3 & 16 & {--} & 3.28 \\ 
	\midrule
	\emph{mAR-SCF} (Ours) & Affine & 3 & 32 & 256 & 1.04  & 3 & 32 & 256 & 3.33  \\
	\emph{mAR-SCF} (Ours) & Affine & 3 & 32 & 512 & 1.03 & 3 & 32 & 512 & 3.31  \\
	\emph{mAR-SCF} (Ours) & MixLogCDF & 3 & 4 & 96 & \bfseries 0.88 & 3 & 4 & 96 & 3.27  \\
	\emph{mAR-SCF} (Ours) & MixLogCDF & {--} & {--} & {--} & {--} & 3 & 4 & 256 & \bfseries 3.24 \\
	\bottomrule
  \end{tabularx}
  \caption{Evaluation of our \emph{mAR-SCF} model on MNIST and CIFAR10 (using uniform dequantization for fair comparsion with \cite{abs-1906-02735,KingmaD18}).}
  \label{tab:cm_abl}
\end{table*}

\myparagraph{Analysis of sampling time.} 
We now formally analyze the computational cost in the number of steps $T$ required for sampling with our \emph{mAR-SCF} model. 
First, we describe the sampling process in detail in \cref{algo:samp} (the forward training process follows the sampling process in reverse order). 
Next, we derive the worst-case number of steps $T$ required by \marps{}, given sufficient parallel resources to sample a channel in parallel. Here, the number of steps $T$ can be seen as the length of the critical path while sampling.
\begin{lemma}
Let the sampled image $\mathbf{x}$ be of resolution $\left[C,N,N\right]$, then the worst-case number of steps $T$ (length of the critical path) required by MARPS is $\mathcal{O}(N)$.
\end{lemma}
 
\begin{proof}
%
At the first sampling step (\cref{fig:gen_model}) at layer $f_{\theta_n}$, our \emph{mAR} prior is applied to generate $\mathbf{h}_{n}$, which is of shape $[2^{n+1} \, C, N/2^{n}, N/2^n]$. 
Therefore, the number of sequential steps required at \emph{the last flow layer} $\mathbf{h}_n$ is
\begin{align}
    T_{n} = C \cdot 2^{n + 1}.
\end{align}
Here, we are assuming that each channel can be sampled in parallel in one time-step.

From $f_{\theta_{n-1}}$ to $f_{\theta_1}$, $f_{\theta_i}$ always contains a \osplit{} operation. Therefore, at each $f_{\theta_i}$ we use our \emph{mAR} prior to sample $\mathbf{l}_{i}$, which has shape  $[2^{i} \, C, N/2^{i}, N/2^i]$. 
Therefore, the number of sequential steps required for sampling at layers $\mathbf{h}_i, 1\leq i < n$ of our \emph{mAR-SCF} model is
\begin{align}
    T_{i} = C \cdot 2^{i}.
\end{align}

Therefore, the total number of sequential steps (length of the critical path) required for sampling is
\begin{align}
    \begin{split}
        T &= T_{n} + T_{n-1} + \cdots + T_{i} + \cdots + T_{1}\\
          &= C \cdot \big( 2^{n + 1} + 2^{n-1} + \cdots + 2^{i} + \cdots + 2^{1} \big) \\
          &= C \cdot \big( 3 \cdot 2^{n} - 2 \big).
    \end{split}
\end{align}

Now, the total number of layers in our \emph{mAR-SCF} model is $n \leq \log(N)$. This is because each layer reduces the spatial resolution by a factor of two. Therefore, the total number of time-steps required is
\begin{align}
    \begin{split}
        T &\leq 3 \cdot C \cdot N.
    \end{split}
\end{align}

In practice, $C \ll N$, with $C=\text{C}_{0}=3$ for RGB images.
Therefore, the total number of sequential steps required for sampling in our \emph{mAR-SCF} model is  $T = \mathcal{O}(N)$.
\end{proof}

It follows that with our multi-scale autoregressive \emph{mAR} priors in our \emph{mAR-SCF} model, sampling can be performed in a linear number of time-steps in contrast to fully autoregressive models like PixelCNN, which require a quadratic number of time-steps \cite{OordKK16}.

\myparagraph{Interpolation.}
A major advantage of invertible flow-based models is that they allow for latent spaces, which are useful for downstream tasks like interpolation -- smoothly transforming one data point into another. 
Interpolation is simple in case of typical invertible flow-based models, because the latent space is modeled as a unimodal \iid Gaussian. 
To allow interpolation in the space of our multimodal \emph{mAR} priors, we develop a simple method based on \cite{BreglerO94}.

Let $\mathbf{x}_\text{A}$ and $\mathbf{x}_\text{B}$ be the two images (points) to be interpolated and $\mathbf{z}_\text{A}$ and $\mathbf{z}_\text{B}$ be the corresponding points in the latent space. We begin with an initial linear interpolation between the two latent points, $\left\{ \mathbf{z}_\text{A}, \mathbf{z}^{1}_\text{A,B}, \cdots, \mathbf{z}^{k}_\text{A,B}, \mathbf{z}_\text{B} \right\}$, such that, $\mathbf{z}^{i}_\text{A,B} = (1 - \alpha^i) \, \mathbf{z}_\text{A} + \alpha^i \, \mathbf{z}_\text{B}$. 
The initial linearly interpolated points $\mathbf{z}^{i}_\text{A,B}$ may not lie in a high-density region under our multimodal prior, leading to non-smooth transformations. 
Therefore, we next project the interpolated points $\mathbf{z}^{i}_\text{A,B}$ to a high-density region, without deviating too much from their initial position. 
This is possible because our \emph{mAR} prior allows for exact inference. 
However, the image corresponding to the projected $\bar{\mathbf{z}}^{i}_\text{A,B}$ must also not deviate too far from either $\mathbf{x}_\text{A}$ and $\mathbf{x}_\text{B}$ either to allow for smooth transitions. 
To that end, we define the projection operation as
\begin{align}
    \begin{split}
    \bar{\mathbf{z}}^{i}_\text{A,B} &= \arg\min \Big(\big\lVert \bar{\mathbf{z}}^{i}_\text{A,B} - \mathbf{z}^{i}_\text{A,B} \big\lVert - \lambda_1 \log p_{\phi}\big(\bar{\mathbf{z}}^{i}_\text{A,B}\big) \\
    &+ \lambda_2 \min\big( \big\lVert f^{-1}(\bar{\mathbf{z}}^{i}_\text{A,B}) - \mathbf{x}_\text{A} \big\lVert, \big\lVert f^{-1}(\bar{\mathbf{z}}^{i}_\text{A,B}) - \mathbf{x}_\text{B} \big\lVert  \big) \Big),
    \end{split}\raisetag{38pt}
    \label{eq:interpolation}
\end{align}
where $\lambda_1,\lambda_2$ are the regularization parameters. 
The term controlled by $\lambda_1$ pulls the interpolated $\mathbf{z}^{i}_\text{A,B}$ back to high-density regions, while the term controlled by $\lambda_2$ keeps the result close to the two images $\mathbf{x}_\text{A}$ and $\mathbf{x}_\text{B}$. Note that this reduces to linear interpolation when $\lambda_1 = \lambda_2=0$. 

\section{Experiments}
We evaluate our approach on the MNIST \cite{lecun1998gradient}, CIFAR10 \cite{krizhevsky2009learning}, and ImageNet \cite{OordKK16} datasets.
In comparison to datasets like CelebA, CIFAR10 and ImageNet are highly multimodal and the performance of invertible SCF models has lagged behind autoregressive models in density estimation and behind GAN-based generative models regarding image quality. 

\subsection{MNIST and CIFAR10}

\begin{figure*}[t]
\centering
	\begin{subfigure}[b]{0.5\textwidth}
        \centering
        \includegraphics[height=5cm]{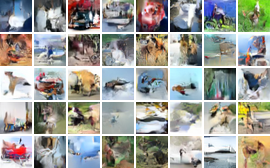}
        \caption{Residual Flows \cite{abs-1906-02735} (3.28 bits/dim, 46.3 FID)}
        \label{fig:pr_anlya}
    \end{subfigure}\hfill%
    \begin{subfigure}[b]{0.5\textwidth}
        \centering
        \includegraphics[height=5cm]{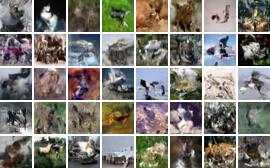}
        \caption{Flow++ with variational dequantization \cite{HoCSDA19} (3.08 bits/dim)}
        \label{fig:pr_anlyb}
    \end{subfigure}
    
    \begin{subfigure}[b]{0.5\textwidth}
        \centering
        \includegraphics[height=5cm]{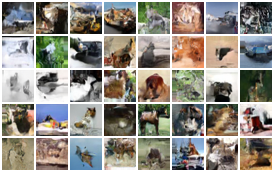}
        \caption{Our \emph{mMAR-SCF} Affine (3.31 bits/dim, 41.0 FID)}
        \label{fig:pr_anlyc}
    \end{subfigure}%
    \begin{subfigure}[b]{0.5\textwidth}
        \centering
        \includegraphics[height=5cm]{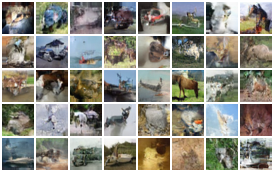}
        \caption{Our \emph{mMAR-SCF} MixLogCDF (3.24 bits/dim, 41.9 FID)}
        \label{fig:pr_anlyd}
    \end{subfigure}

\caption{Comparison of random samples from our \emph{mAR-SCF} model with state-of-the-art models. }
\label{fig:cifar10_samp}
\end{figure*}

\paragraph{Architecture details.} 
Our \emph{mAR} prior at each level $f_{\theta_i}$ consists of three convolutional LSTM layers, each of which uses 32 convolutional filters to compute the input-to-state and state-to-state components. 
Keeping the \emph{mAR} prior architecture constant, we experiment with different SCF couplings in $f_{\theta_i}$ to highlight the effectiveness of our \emph{mAR} prior. 
We experiment with affine couplings of \cite{DinhSB17,KingmaD18} and MixLogCDF couplings \cite{HoCSDA19}. 
Affine couplings have limited modeling flexibility. 
The more expressive MixLogCDF applies the cumulative distribution function of a mixture of logistics. 
In the following, we include experiments varying the number couplings and the number of channels in the convolutional blocks of the neural networks used to predict the affine/MixLogCDF transformation parameters.\footnotemark[2] \footnotetext[2]{Code available at: \url{https://github.com/visinf/mar-scf}}

\myparagraph{Hyperparameters.} 
We use Adamax (as in \cite{KingmaD18}) with a learning rate of $8 \times 10^{-4}$. 
We use a batch size of 128 with affine and 64 with MixLogCDF couplings (following \cite{HoCSDA19}).

\myparagraph{Density estimation.} 
We report density estimation results on MNIST and CIFAR10 in \cref{tab:cm_abl} using the per-pixel log-likelihood metric in bits/dim. 
We also include the architecture details (\# of levels, coupling type, \# of channels). 
We compare to the state-of-the-art Flow++ \cite{HoCSDA19} method with SCF couplings and Residual Flows \cite{abs-1906-02735}. 
Note that in terms of architecture, our \emph{mAR-SCF} model with affine couplings is closest to that of Glow \cite{KingmaD18}. 
Therefore, the comparison with Glow serves as an ideal ablation to assess the effectiveness of our \emph{mAR} prior. 
Flow++ \cite{HoCSDA19}, on the other hand, uses the more powerful MixLogCDF transformations and their model architecture does not include \osplit{} operations. 
Because of this, Flow++ has higher computational and memory requirements for a given batch size compared to Glow. 
Furthermore, for fair comparison with Glow \cite{KingmaD18} and Residual flows \cite{abs-1906-02735}, we use uniform dequantization unlike Flow++, which proposes to use variational dequantization.                 

In comparison to Glow, we achieve improved density estimation results on both MNIST and CIFAR10. 
In detail, we outperform Glow (\eg 1.05 \vs 1.04 bits/dim on MNIST and 3.35 \vs 3.33 bits/dim on CIFAR10) with $\lvert \text{SCF} \lvert=32$ affine couplings and 3 levels, while using parameter prediction networks with only half (256 \vs 512) the number of channels.
We observe that increasing the capacity of our parameter prediction networks to 512 channels boosts the log-likelihood further to 1.03 bits/dim on MNIST and 3.31 bits/dim on CIFAR10.
As this setting with 512 channels is identical to setting reported in \cite{KingmaD18}, this shows that our \emph{mAR} prior boosts the accuracy by $\sim\!0.04$ bits/dim in case of CIFAR10. 
To place this performance gain in context, it is competitive with the $\sim\!0.03$ bits/dim boost reported in \cite{KingmaD18} (\cf Fig.~3 in \cite{KingmaD18}) with the introduction of the $1 \times 1$ convolution. 
We train our model for $\sim\!3000$ epochs, similar to \cite{KingmaD18}. 
Also note that we only require a batch size of 128 to achieve state-of-the-art likelihoods, whereas Glow uses batches of size 512. 
Thus our \emph{mAR-SCF} model improves density estimates and requires significantly lower computational resources ($\sim\!48$ \vs $\sim\!128$ GB memory). 
Overall, we also observe competitive sampling speed (see also \cref{tab:cm_speed}). 
This firmly establishes the utility of our \emph{mAR-SCF} model. 

For fair comparison with Flow++ \cite{HoCSDA19} and Residual Flows \cite{abs-1906-02735}, we employ the more powerful MixLogCDF couplings. Our \emph{mAR-SCF} model uses 4 MixLogCDF couplings at each level with 96 channels but includes \osplit{} operations unlike Flow++. Here, we outperform Flow++ and Residual Flows (3.27 \vs 3.29 and 3.28 bits/dim on CIFAR10) while being equally fast to sample as Flow++ (\cref{tab:cm_speed}). A baseline model without our \emph{mAR} prior has performance comparable to Flow++ (3.29 bits/dim).
Similarly on MNIST, our \emph{mAR-SCF} model again outperforms Residual Flows (0.88 \vs 0.97 bits/dim).
Finally, we train a more powerful \emph{mAR-SCF} model with 256 channels with sampling speed competitive with \cite{abs-1906-02735}, which achieves state-of-the-art 3.24 bits/dim on CIFAR10 \footnotemark[3]. This is attained after $\sim\!400$ training epochs (comparable to $\sim\!350$ epochs required by \cite{abs-1906-02735} to achieve 3.28 bits/dim).  
Next, we compare the sampling speed of our \emph{mAR-SCF} model with that of Flow++ and Residual Flow.

\begin{table}[h]
  \centering
  \scriptsize
  \begin{tabularx}{\columnwidth}{@{}Xl@{\hskip 1em}S[table-format=1]@{\hskip 1em}S[table-format=2]@{\hskip 1em}S[table-format=3]@{\hskip 2.5em}S[table-format=2,table-figures-exponent = 1]@{}}
	\toprule
	Method & Coupling & {Levels} & {$\lvert \text{SCF} \lvert$} & {Ch.} & {Speed $(\si{\milli\s}, \downarrow)$}\\
	\midrule
	Glow \cite{KingmaD18} & Affine & 3 & 32 & 512 & 13  \\ 
	Flow++ \cite{HoCSDA19} & MixLogCDF & 3 & {--} & 96 & 19\\ 
	Residual Flow \cite{abs-1906-02735} &  Residual & 3 & 16 & {--} & 34 \\ 
	\midrule
	PixelCNN++ \cite{SalimansK0K17} & Autoregressive & {--} & {--} & {--} & 5e3 \\ 
	\midrule
	\emph{mAR-SCF} (Ours) &  Affine & 3 & 32 & 256 &  \bfseries 6 \\
	\emph{mAR-SCF} (Ours) &  Affine & 3 & 32 & 512 &  17 \\
	\emph{mAR-SCF} (Ours) &  MixLogCDF & 3 & 4 & 96  & 19 \\
	\emph{mAR-SCF} (Ours) &  MixLogCDF & 3 & 4 & 256 & 32 \\
	\bottomrule
  \end{tabularx}
  \caption{Evaluation of sampling speed with batches of size 32.}
  \label{tab:cm_speed}
  \vspace{-0.4 em}
\end{table}
\footnotetext[3]{\emph{mAR-SCF} with 256 channels trained on CIFAR10($\sim 1000$ epochs) trained to convergence achieves 3.22 bits/dim}
\myparagraph{Sampling speed.} 
We report the sampling speed of our \emph{mAR-SCF} model in \cref{tab:cm_speed} in terms of sampling one image on CIFAR10. 
We report the average over 1000 runs using a batch size of 32. 
We performed all tests on a single Nvidia V100 GPU with 32GB of memory. 
First, note that our \emph{mAR-SCF} model with affine coupling layers in 3 levels with 512 channels needs \SI{17}{\milli\s} on average to sample an image. 
This is comparable with Glow, which requires \SI{13}{\milli\s}. 
This shows that our \emph{mAR} prior causes only a slight increase in sampling time -- particularly because our \emph{mAR-SCF} requires only $\mathcal{O}(N)$ steps to sample and the prior has far fewer parameters compared to the invertible flow network. 
Moreover, our \emph{mAR-SCF} model with affine coupling layers with 256 channels is considerably faster (6 \vs \SI{13}{\milli\s}) with an accuracy advantage. 
Similarly, our \emph{mAR-SCF} with MixLogCDF and 96 channels is competitive in speed with \cite{HoCSDA19} with an accuracy advantage and considerably faster than \cite{abs-1906-02735} (19 \vs \SI{34}{\milli\s}). 
This is because Residual Flows are slower to invert (sample) as there is no closed-form expression of the inverse. 
Furthermore, our \emph{mAR-SCF} with MixLogCDF and 256 channels is competitive with respect to \cite{abs-1906-02735} in terms of sampling speed while having a large accuracy advantage. 
Finally, note that these sampling speeds are two orders of magnitude faster than state-of-the-art fully autoregressive approaches, \eg PixelCNN++ \cite{SalimansK0K17}.

\begin{figure}[t]
\centering
    \begin{subfigure}[b]{\linewidth}
        \centering
        \includegraphics[width=\linewidth]{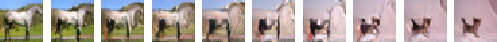}
    \end{subfigure}\\
    \begin{subfigure}[b]{\linewidth}
        \centering
        \includegraphics[width=\linewidth]{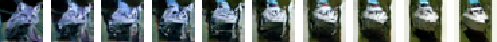}
    \end{subfigure}
\caption{Interpolations of our \emph{mAR-SCF} model on CIFAR10.}
\label{fig:interpolation}
\vspace{-0.5em}
\end{figure}

\begin{figure}[t]
\centering
    \begin{subfigure}{\columnwidth}
        \centering
        \includegraphics[height=3.55cm]{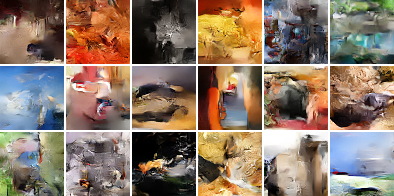}
        \caption{Residual Flows \cite{abs-1906-02735} (3.75 bits/dim)}
        \label{fig:imagenet_sampa}
    \end{subfigure}%
    
    \begin{subfigure}{\columnwidth}
        \centering
        \includegraphics[height=3.55cm]{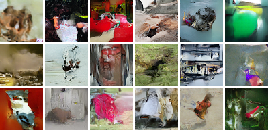}
        \caption{Our \emph{mMAR-SCF} (Affine, 3.80 bits/dim) }
        \label{fig:imagenet_sampb}
    \end{subfigure}

\caption{Random samples on ImageNet ($64 \times 64$).}
\label{fig:imagenet_samp}
\vspace{-0.5em}
\end{figure}

\myparagraph{Sample quality.} 
Next, we analyze the sample quality of our \emph{mAR-SCF} model in \cref{tab:cm_fid} using the FID metric \cite{HeuselRUNH17} and Inception scores \cite{SalimansGZCRCC16}. 
The analysis of sample quality is important as it is well-known that visual fidelity and test log-likelihoods are not necessarily indicative of each other \cite{TheisOB15}.
We achieve an FID of 41.0 and an Inception score of 5.7 with our \emph{mAR-SCF} model with affine couplings, significantly better than Glow with the same specifications and Residual Flows. While our \emph{mAR-SCF} model with MixLogCDF couplings also performs comparably, empirically we find affine couplings to lead to better image quality as in \cite{abs-1906-02735}.
We show random samples from our \emph{mAR-SCF} model with both affine and MixLogCDF couplings in \cref{fig:cifar10_samp}. 
Here, we compare to the version of Flow++ with MixLogCDF couplings and variational dequantization (which gives even better log-likelihoods) and Residual Flows. 
Our \emph{mAR-SCF} model achieves better sample quality with more clearly defined objects.
Furthermore, we also obtain improved sample quality over both PixelCNN and PixelIQN and close the gap in comparison to adversarial approaches like DCGAN \cite{RadfordMC15} and WGAN-GP \cite{WeiGL0W18}.
This highlights that our \emph{mAR-SCF} model is able to better capture long-range correlations. 

\myparagraph{Interpolation.} We show interpolations on CIFAR10 in \cref{fig:interpolation}, obtained using \cref{eq:interpolation}. We observe smooth interpolation between images belonging to distinct classes. 
This shows that the latent space of our \emph{mAR} prior can be potentially used for downstream tasks similarly to Glow \cite{KingmaD18}. We include additional analyses in Appendix E.

\begin{table}[t]
  \centering
  \footnotesize
  \begin{tabularx}{\columnwidth}{@{}Xl@{\hskip 2em}S[table-format=2.1]S[table-format=1.1]@{}}
	\toprule
	Method & Coupling & {FID $(\downarrow)$} & {Inception Score $(\uparrow)$}\\
	\midrule
	PixelCNN \cite{OordKK16} & Autoregressive & 65.9 & 4.6\\
	PixelIQN  \cite{OstrovskiDM18} & Autoregressive & 49.4 & {--}\\
	\midrule
	Glow \cite{KingmaD18} & Affine & 46.9 & {--} \\ 
	Residual Flow \cite{abs-1906-02735} & Residual & 46.3 & 5.2 \\
	\emph{mAR-SCF} (Ours) & MixLogCDF &  41.9 & \bfseries 5.7 \\
	\emph{mAR-SCF} (Ours) & Affine & \bfseries 41.0 & \bfseries 5.7 \\
	\midrule
	DCGAN \cite{RadfordMC15} & Adversarial & 37.1 & 6.4\\
    WGAN-GP \cite{WeiGL0W18} & Adversarial & \itshape 36.4 & \itshape 6.5\\
	\bottomrule
  \end{tabularx}
  \caption{Evaluation of sample quality on CIFAR10. Other results are quoted from \cite{abs-1906-02735,OstrovskiDM18}.}
  \label{tab:cm_fid}
  \vspace{-0.5em}
\end{table}

\begin{table}[t]
  \centering
  \scriptsize
  \begin{tabularx}{\columnwidth}{@{}X@{\hskip 1em}l@{\hskip 1em}S[table-format=2]S[table-format=3]@{\hskip 2em}S[table-format=1.2]r@{}S[table-format=3,table-space-text-post={$\quad\:$}]@{}}
	\toprule
	Method & Coupling & {$\lvert \text{SCF} \lvert$} & {Ch.} & {bits/dim $(\downarrow)$} & \multicolumn{2}{@{}c@{}}{Mem (GB, $\downarrow$)}\\
	\midrule
	Glow \cite{KingmaD18} & Affine & 32 & 512 & 4.09 & $\:\sim$ & 128 \\
	Residual Flow \cite{abs-1906-02735} & Residual & 32 & {--} & 4.01 & \multicolumn{2}{@{}c@{}}{--} \\
	\midrule
	\emph{mAR-SCF} (Ours) & Affine & 32 & 256 & 4.07 & $\sim$ & \bfseries 48 \\
	\emph{mAR-SCF} (Ours) & MixLogCDF & 4 & 460 & \bfseries 3.99 & $\sim$ & 80 \\
	\bottomrule
  \end{tabularx}
  \caption{Evaluation on ImageNet ($32 \times 32$).}
  \label{tab:cm_fid_imagenet}
\end{table}

\subsection{ImageNet}
Finally, we evaluate our \emph{mAR-SCF} model on ImageNet ($32 \times 32$ and $64 \times 64$) against the best performing models on MNIST and CIFAR10 in \cref{tab:cm_fid_imagenet}, \ie{} Glow \cite{KingmaD18} and Residual Flows \cite{abs-1906-02735}. 
Our model with affine couplings outperforms Glow while using fewer channels (4.07 \vs 4.09 bits/dim).
For comparison with the more powerful Residual Flow models, we use four MixLogCDF couplings at each layer $f_{\theta_i}$ with 460 channels. 
We again outperform Residual Flows \cite{abs-1906-02735} (3.99 \vs 4.01 bits/dim). 
These results are consistent with the findings in \cref{tab:cm_abl}, highlighting the advantage of our \emph{mAR} prior. 
Finally, we also evaluate on the ImageNet ($64 \times 64$) dataset. 
Our \emph{mAR-SCF} model with affine flows achieves 3.80 \vs 3.81 bits/dim in comparison to Glow \cite{KingmaD18}. 
We show qualitative examples in \cref{fig:imagenet_samp} and compare to Residual Flows. 
We see that although the powerful Residual Flows obtain better log-likelihoods (3.75 bits/dim), our \emph{mAR-SCF} model achieves better visual fidelity. 
This again highlights that our \emph{mAR} is able to better capture long-range correlations.

\section{Conclusion}
We presented \emph{mAR-SCF}, a flow-based generative model with novel multi-scale autoregressive priors for modeling long-range dependencies in the latent space of flow models. 
Our \emph{mAR} prior considerably improves the accuracy of flow-based models with split coupling layers. 
Our experiments show that not only does our \emph{mAR-SCF} model improve density estimation (in terms of bits/dim), but also considerably improves the sample quality of the generated images compared to previous state-of-the-art exact inference models. 
We believe the combination of complex priors with flow-based models, as demonstrated by our \emph{mAR-SCF} model, provides a path toward efficient models for exact inference that approach the fidelity of GAN-based approaches.

\clearpage
{\small 
\myparagraph{Acknowledgement.} SM and SR acknowledge the support by the German Research Foundation as part of the Research Training Group \emph{Adaptive Preparation of Information from Heterogeneous Sources (AIPHES)} under grant No.~GRK 1994/1.}
{\small
\bibliographystyle{ieee_fullname}
\bibliography{egbib}
}

\clearpage
\appendix
\appendixpage
\addappheadtotoc
We provide additional details of Lemma 4.1 in the main paper, additional details of our \emph{mAR-SCF} model architecture, as well as additional results and qualitative examples.  
\begin{figure*}[t]
\centering
   \includegraphics[width=0.6\linewidth]{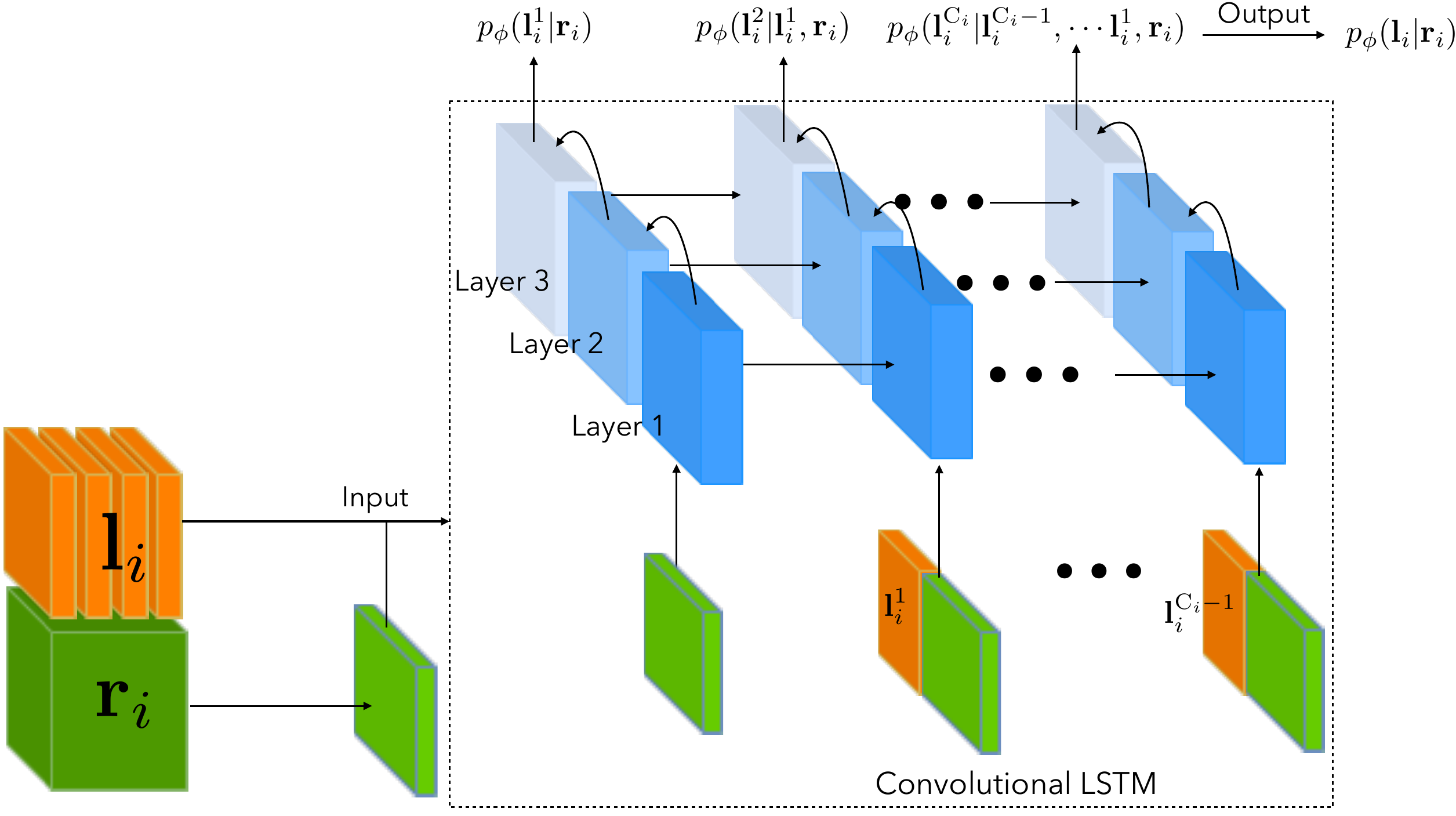}
   \caption{Architecture of our multi-scale autoregressive (\emph{mAR}) prior.}
\label{fig:mARarchi}
\end{figure*}
\section{Channel Dimensions in \emph{mAR-SCF}}
We begin by providing additional details of Lemma 4.1 in the main paper. We formally prove the claim that the number of channels at the last layer $n$ (which does not have a \osplit{} operation) is $\text{C}_n = 2^{n+1}\cdot C$ for an image of size $[C, N, N]$. Thereby, we also show that the number of channels at any layer $i$ (with a \osplit{} operation) is $\text{C}_i = 2^{i} \cdot C$.

Note that the number of time steps required for sampling depends on the number of channels for \emph{mAR-SCF} (flow model) based on \emph{MARPS} (Algorithm 1). 
\begin{lemma}
Let the size of the sampled image be $[C, N,N ]$. The number of channels at the last layer $\mathbf{h}_n$ is $\text{C}_n = 2^{n+1}\cdot C$.
\end{lemma}
\begin{proof}
The shape of image to be sampled is $[C, N, N]$. Since the network is invertible, the size of the image at level $\mathbf{h}_0$ is the size of the image sampled, \ie $[C, N, N]$.

First, let the number of layers $n$ be 1. 
This implies that during the forward pass the network applies a \squeeze{} operation, which reshapes the image to $[4C, N/2, N/2]$ followed by \sflow{}, which does not modify the shape of the input, \ie the shape remains $[4C, N/2, N/2]$. 
Thus the number of channels at the last layer $\mathbf{h}_1$ is $\text{C}_1 = 4 \cdot C = 2^{1+1}\cdot C$ when $n=1$.

Next, let us assume that the number of channels at the last layer $\mathbf{h}_{k-1}$ for a flow with $k-1$ layers is
\begin{equation}
    \text{C}_{k-1} = 2^{k}\cdot C.
    \label{eq:mi_assume}
\end{equation}
We aim to prove by induction that the number of channels at the last layer $\mathbf{h}_k$ for a flow with $k$ layers then is
\begin{equation}
    \text{C}_{k} = 2^{k+1}\cdot C.
    \label{eq:mi_hypo}
\end{equation}

To that end, we note that the dimensionality of the output at layer $k-1$ after \squeeze{} and $\sflow{}$ from \cref{eq:mi_assume} by assumption is $\text{C}_{k-1} = 2^{k} \cdot C$. 
For a flow with $k$ layers, the ${(k-1)}^\text{st}$ layer has a \osplit{} operation resulting in  $\{\mathbf{l}_{k-1},\mathbf{r}_{k-1}\}$, each with size $[2^{k-1}C, N/2^{k-1},N/2^{k-1} ]$. 
The number of channels for $\mathbf{r}_{k-1}$ at layer $k-1$ is thus $2^{k}\cdot C/2 = 2^{k-1} \cdot C$. 
At layer $k$ the input with $2^{k-1} \cdot C$ channels is transformed by \squeeze{} to $2^{k-1}\cdot 4\cdot C = 2^{(k-1)+2}\cdot C = 2^{k+1}\cdot C$ channels.

Therefore, by induction \cref{eq:mi_hypo} holds for $k=n$. Thus the number of channels at the last layer $\mathbf{h}_n$ is given as $\text{C}_n= 2^{n+1}\cdot C$.
\end{proof}

\section{\emph{mAR-SCF} Architecture}
In \cref{fig:mARarchi}, we show the architecture of our \emph{mAR} prior in detail for a layer $i$ of \emph{mAR-SCF}. The network is a convolutional LSTM with three LSTM layers. The input at each time-step is the previous channel of $\mathbf{l}_i$, concatenated with the output after convolution on $\mathbf{r}_i$. Our \emph{mAR} prior autoregressivly outputs the probability of each channel ${\mathbf{l}}^{j}_{i}$ given by the distribution $p_\phi\big(\mathbf{l}_i^{j}|\mathbf{l}_i^{j-1}, \cdots,\mathbf{l}_i^{1},\mathbf{r}_i \big)$ modeled as $\mathcal{N}\big(\mu^{j}_{i}, \sigma^{j}_{i}\big)$ during inference. Because of the internal state of the Convolutional LSTM (memory), our \emph{mAR} prior can learn long-range dependencies. 

In Fig.~2 in the main paper, the \sflow{} operation consists of an activation normalization layer, an invertible $1 \times 1$ convolution layer followed by split coupling layers (32 layers for affine couplings or 4 layers of MixLogCDF at each level). 

\section{Empirical Analysis of Sampling Speed}
In \cref{fig:cifar10_time}, we analyze the real-world sampling time of our state-of-the-art \emph{mAR-SCF} model with MixLogCDF couplings using varying image sizes $[C,N,N]$. 
In particular, we vary the input spatial resolution $N$. 
We report the mean and variance of 1000 runs with batch size of 8 on an Nvidia V100 GPU with 32GB memory. 
Using smaller batch sizes of 8 ensures that we always have enough parallel resources for all image sizes. 
Under these conditions, we see that the sampling time increases linearly with the image size.
This is because the number of sampling steps of our \emph{mAR-SCF} model scales as $\mathcal{O}(N)$, \cf Lemma 4.1.

\begin{figure}[h]
\centering
\includegraphics[height=4.8cm]{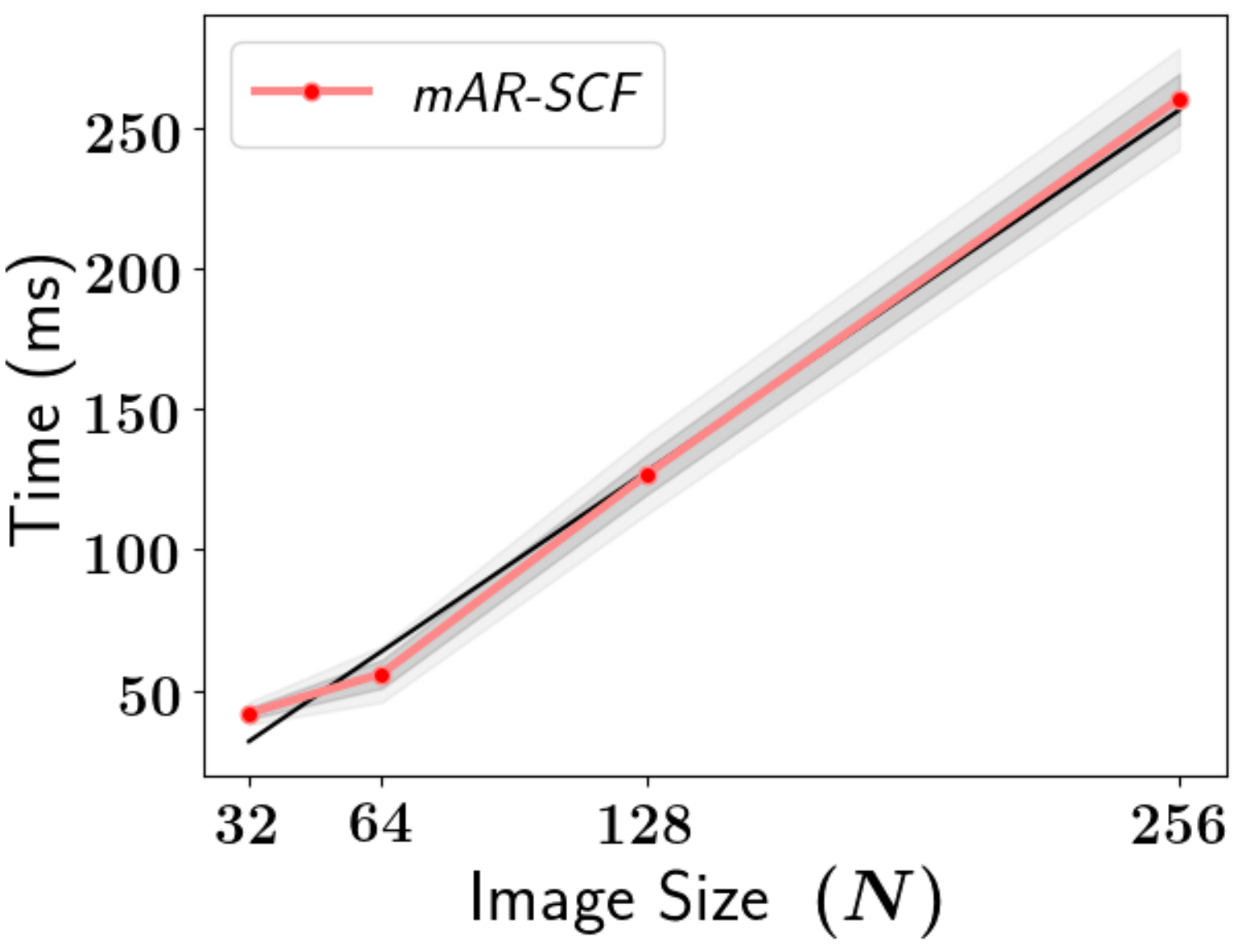}
\caption{Analysis of real-world sampling time of our \emph{mAR-SCF} model with varying image size $[C,N,N]$. For reference, we show the line $y = 2x$ in black.}
\label{fig:cifar10_time}
\end{figure}

\section{Additional Qualitative Examples}
We include additional qualitative examples for training our generative model on the MNIST, CIFAR10, and ImageNet ($32 \times 32$) datasets. 
In \cref{fig:mnist_samp}, we see that the visual sample quality of our \emph{mAR-SCF} model with MixLogCDF couplings is competitive with those of the state-of-the-art Residual Flows \cite{abs-1906-02735}. 
Moreover, note that we achieve better test log-likelihoods (0.88 \vs 1.00 bits/dim).

We additionally provide qualitative examples for ImageNet ($32 \times 32$) in \Cref{fig:imagenet32_samp}. We observe that in comparison to the state-of-the-art Flow++ \cite{HoCSDA19} and Residual Flows \cite{abs-1906-02735}, the images generated by our \emph{mAR-SCF} model are detailed with a competitive visual quality.

Finally, on CIFAR10 we compare with the fully autoregressive PixelCNN model \cite{OordKK16} in \Cref{fig:cifar10_2_samp}. We see that although the fully autoregressive PixelCNN achieves significantly better test log-likelihoods (3.00 \vs 3.24 bits/dim), our \emph{mAR-SCF} models achieves better visual sample quality (also shown by the FID and Inception metrics in Table 3 of the main paper).

\section{Additional Interpolation Results}
\begin{figure}[h]
        \centering
        \includegraphics[height=6cm]{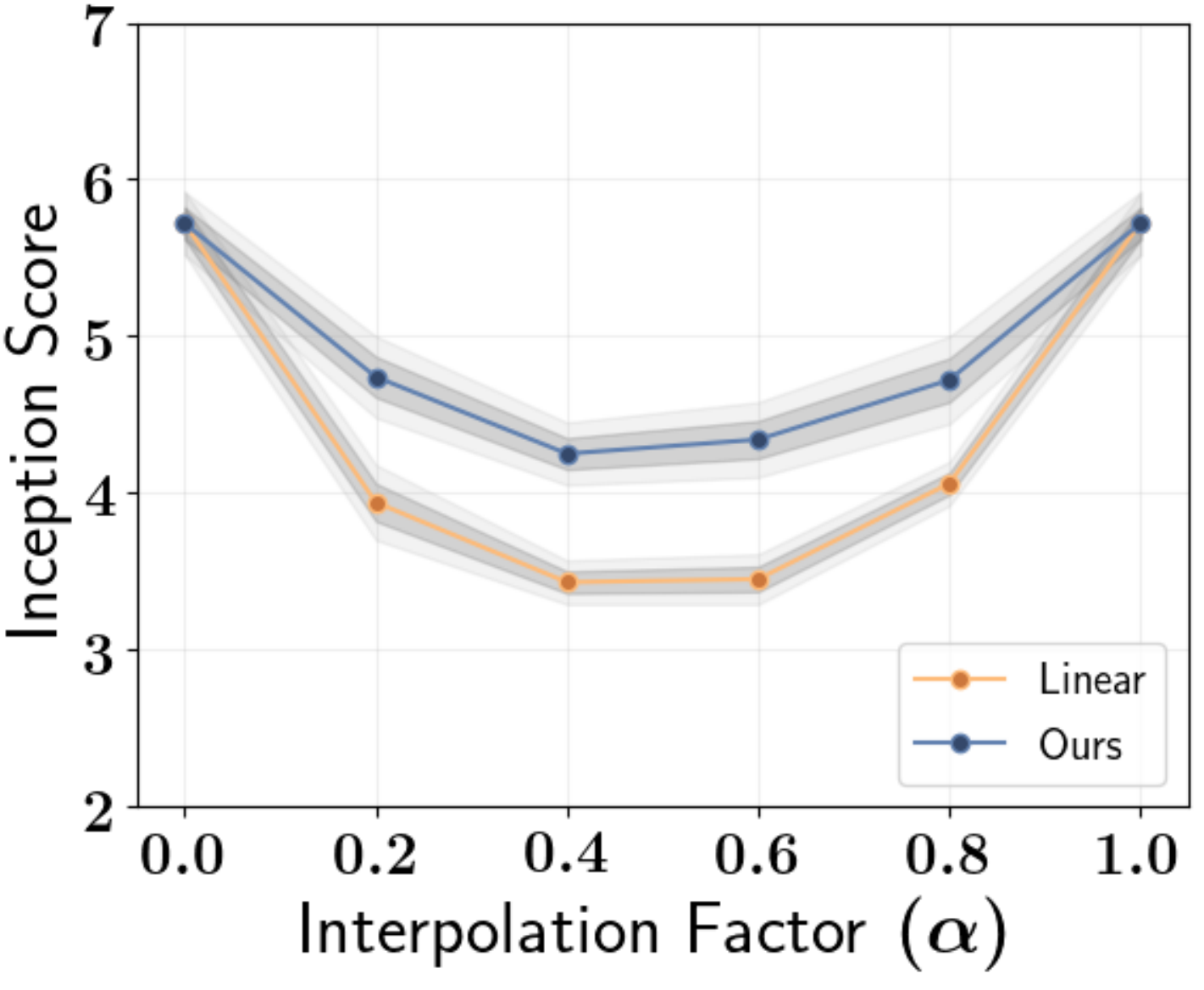}
\caption{Inception scores of random samples generated with our interpolation scheme versus linear interpolation.}
\label{fig:cifar10_interp}
\end{figure}
Finally, in \cref{fig:interp} we show qualitative examples to compare the effect using our proposed interpolation method (Eq.~11) versus simple linear interpolation in the multimodal latent space of our \emph{mAR} prior.
We use the Adamax optimizer to optimize Eq.~(11) with a learning rate of $5\times10^{-2}$ for $\sim$~100 iterations. The computational requirement per iteration is approximately equivalent to a standard training iteration of our  \emph{mAR-SCF} model, \ie $\sim$ 1 sec for a batch of size 128.
We observe that interpolated images obtained from our \emph{mAR-SCF} affine model using our proposed interpolation method (Eq.~11) have a better visual quality, especially for the interpolations in the middle of the interpolating path. 
Note that we find our scheme to be stable in practice in a reasonably wide range of the hyperparameters $\lambda_2 \approx \lambda_1\in[0.2,0.5]$ as it interpolates in high-density regions between $\mathbf{x}_{\text{A}},\mathbf{x}_{\text{B}}$. 
We support the better visual quality of the interpolations of our scheme compared to the linear method with Inception scores computed for interpolations obtained from both methods in \cref{fig:cifar10_interp}.

\begin{figure*}[t]
    \begin{subfigure}{0.49\textwidth}
        \centering
        \includegraphics[height=7cm]{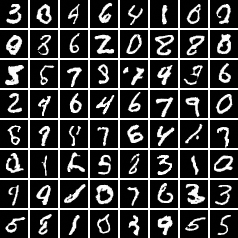}
        \caption{Residual Flows \cite{abs-1906-02735}}
        \label{fig:mnist_sampa}
    \end{subfigure}%
    \begin{subfigure}{0.49\textwidth}
        \centering
        \includegraphics[height=7cm]{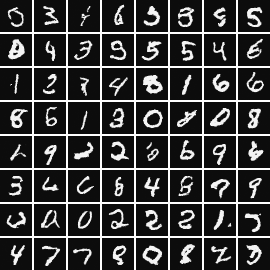}
        \caption{Our \emph{mAR-SCF} (MixLogCDF)}
        \label{fig:mnist_sampb}
    \end{subfigure}

\caption{Random samples when trained on MNIST.}
\label{fig:mnist_samp}
\end{figure*}
\begin{figure*}[h]
\centering
    \begin{subfigure}{0.49\textwidth}
        \centering
        \includegraphics[height=7cm]{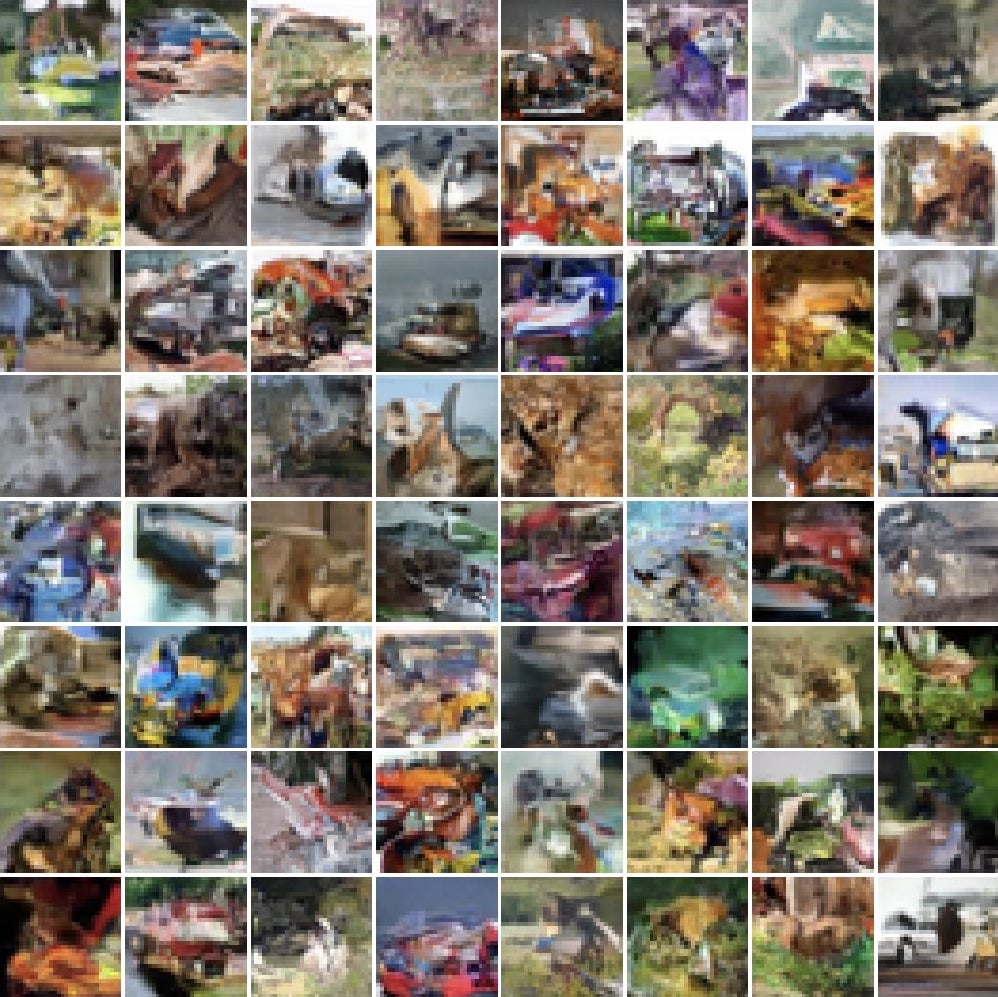}
        \caption{PixelCNN \cite{OordKK16}}
        \label{fig:cifar10_2_sampa}
    \end{subfigure}%
    \begin{subfigure}{0.49\textwidth}
        \centering
        \includegraphics[height=7cm]{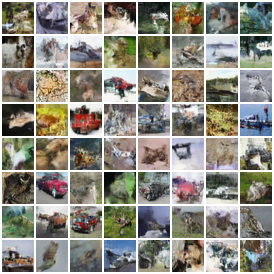}
        \caption{Our \emph{mAR-SCF} (MixLogCDF) }
        \label{fig:cifar10_2_sampb}
    \end{subfigure}

\caption{Random samples when trained on CIFAR10 ($32 \times 32$).}
\label{fig:cifar10_2_samp}
\end{figure*}

\begin{figure*}[h]
\centering
    \begin{subfigure}{0.49\textwidth}
        \centering
        \includegraphics[height=7cm]{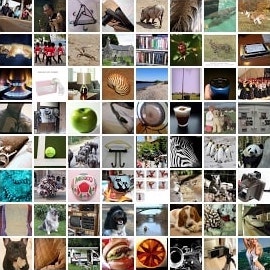}
        \caption{Real Data}
        \label{fig:imagenet32_sampa}
    \end{subfigure}%
    \begin{subfigure}{0.49\textwidth}
        \centering
        \includegraphics[height=7cm]{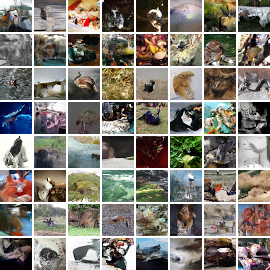}
        \caption{Flow++ \cite{HoCSDA19}}
        \label{fig:imagenet32_sampb}
    \end{subfigure}
    \vspace{2mm}
    
    \begin{subfigure}{0.49\textwidth}
        \centering
        \includegraphics[height=7cm]{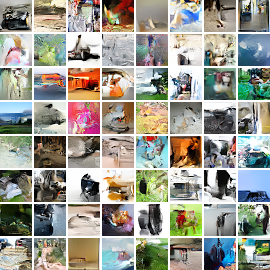}
        \caption{Residual Flows \cite{abs-1906-02735}}
        \label{fig:imagenet32_sampc}
    \end{subfigure}%
    \begin{subfigure}{0.49\textwidth}
        \centering
        \includegraphics[height=7cm]{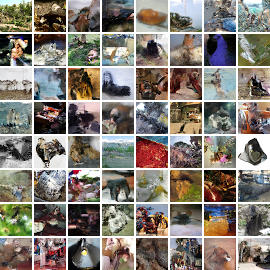}
        \caption{Our \emph{mAR-SCF} (MixLogCDF) }
        \label{fig:imagenet32_sampd}
    \end{subfigure}

\caption{Random samples when trained on ImageNet ($32 \times 32$).}
\label{fig:imagenet32_samp}
\end{figure*}

\begin{figure*}
    \centering
    \begin{tabularx}{1.0\linewidth}{@{}c@{\hskip 0.05cm}c@{\hskip 0.05cm}c@{\hskip 0.05cm}c@{\hskip 0.05cm}c@{\hskip 0.05cm}c@{\hskip 0.5cm}c@{\hskip 0.05cm}c@{\hskip 0.05cm}c@{\hskip 0.05cm}c@{\hskip 0.05cm}c@{\hskip 0.05cm}c@{}}
    
    \includegraphics[height=1.3cm,width=1.3cm]{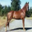} &
    \includegraphics[height=1.3cm,width=1.3cm]{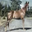} &
    \includegraphics[height=1.3cm,width=1.3cm]{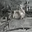} &
    \includegraphics[height=1.3cm,width=1.3cm]{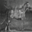} &
    \includegraphics[height=1.3cm,width=1.3cm]{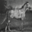} &
    \includegraphics[height=1.3cm,width=1.3cm]{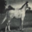} &
    \includegraphics[height=1.3cm,width=1.3cm]{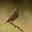} &
    \includegraphics[height=1.3cm,width=1.3cm]{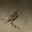} &
    \includegraphics[height=1.3cm,width=1.3cm]{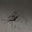} &
    \includegraphics[height=1.3cm,width=1.3cm]{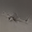} &
    \includegraphics[height=1.3cm,width=1.3cm]{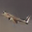} &
    \includegraphics[height=1.3cm,width=1.3cm]{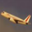} \\
    \vspace{0.5cm}
    
    \includegraphics[height=1.3cm,width=1.3cm]{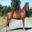} &
    \includegraphics[height=1.3cm,width=1.3cm]{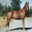} &
    \includegraphics[height=1.3cm,width=1.3cm]{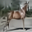} &
    \includegraphics[height=1.3cm,width=1.3cm]{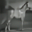} &
    \includegraphics[height=1.3cm,width=1.3cm]{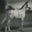} &
    \includegraphics[height=1.3cm,width=1.3cm]{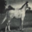} &
    \includegraphics[height=1.3cm,width=1.3cm]{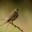} &
    \includegraphics[height=1.3cm,width=1.3cm]{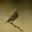} &
    \includegraphics[height=1.3cm,width=1.3cm]{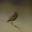} &
    \includegraphics[height=1.3cm,width=1.3cm]{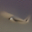} &
    \includegraphics[height=1.3cm,width=1.3cm]{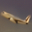} &
    \includegraphics[height=1.3cm,width=1.3cm]{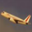}
    \\
    
    \includegraphics[height=1.3cm,width=1.3cm]{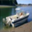} &
    \includegraphics[height=1.3cm,width=1.3cm]{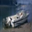} &
    \includegraphics[height=1.3cm,width=1.3cm]{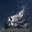} &
    \includegraphics[height=1.3cm,width=1.3cm]{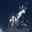} &
    \includegraphics[height=1.3cm,width=1.3cm]{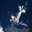} &
    \includegraphics[height=1.3cm,width=1.3cm]{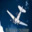} &
    \includegraphics[height=1.3cm,width=1.3cm]{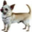} &
    \includegraphics[height=1.3cm,width=1.3cm]{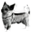} &
    \includegraphics[height=1.3cm,width=1.3cm]{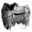} &
    \includegraphics[height=1.3cm,width=1.3cm]{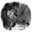} &
    \includegraphics[height=1.3cm,width=1.3cm]{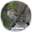} &
    \includegraphics[height=1.3cm,width=1.3cm]{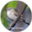} \\
    
    \includegraphics[height=1.3cm,width=1.3cm]{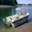} &
    \includegraphics[height=1.3cm,width=1.3cm]{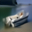} &
    \includegraphics[height=1.3cm,width=1.3cm]{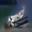} &
    \includegraphics[height=1.3cm,width=1.3cm]{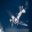} &
    \includegraphics[height=1.3cm,width=1.3cm]{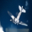} &
    \includegraphics[height=1.3cm,width=1.3cm]{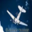} &
    \includegraphics[height=1.3cm,width=1.3cm]{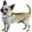} &
    \includegraphics[height=1.3cm,width=1.3cm]{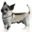} &
    \includegraphics[height=1.3cm,width=1.3cm]{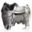} &
    \includegraphics[height=1.3cm,width=1.3cm]{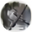} &
    \includegraphics[height=1.3cm,width=1.3cm]{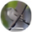} &
    \includegraphics[height=1.3cm,width=1.3cm]{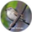}
    \\
    \end{tabularx}
    \caption{{Comparison of interpolation quality of our interpolation scheme (Eq.~11) with a standard linear interpolation scheme. The top row of each result shows the interpolations between real images for the linear method and the corresponding bottom rows are the interpolations with our proposed scheme.}}
    \label{fig:interp}
\end{figure*}

\end{document}